\documentclass[11pt]{article}
\usepackage{amsthm,amsfonts,amsmath,amssymb}
\usepackage{enumerate, graphicx}
\usepackage{xspace}
\usepackage{boxedminipage}
\usepackage{url}
\usepackage{algo}
\usepackage{enumerate, paralist}
\usepackage{accents}
 \usepackage[usenames,dvipsnames]{pstricks}
 \usepackage{epsfig}
 \usepackage{pst-grad} 
 \usepackage{pst-plot} 
 \usepackage{bbm} 
 \usepackage{float} 
\usepackage{hyperref} 
\usepackage{multirow} 
\usepackage{algorithm} 
\usepackage[noend]{algpseudocode} 
\usepackage[bb=boondox]{mathalfa} 
\usepackage{mathtools} 
\usepackage{thmtools, thm-restate}
\usepackage{hyperref}
\usepackage{soul} 

\usepackage{nips_2017}
\usepackage[utf8]{inputenc} 
\usepackage[T1]{fontenc}    
\usepackage{hyperref}       
\usepackage{url}            
\usepackage{booktabs}       
\usepackage{amsfonts}       
\usepackage{nicefrac}       
\usepackage{microtype}      

\def\colorful{1}
\ifnum\colorful=1

\else

\fi

\def\Comment#1{\textsl{$\langle\!\langle$#1\/$\rangle\!\rangle$}}

\newtheorem*{theorem*}{Theorem}

\newtheorem{lemma}{Lemma}[section]
\newtheorem{theorem}[lemma]{Theorem}

\theoremstyle{definition}

\newtheorem{definition}[lemma]{Definition}

\renewenvironment{proof}{\vspace{-0.1in}\noindent{\bf Proof:}}%
        {\hspace*{\fill}$\Box$\par}
        {\hspace*{\fill}$\Box$\par}
        {\hspace*{\fill}$\Box$\par}

\def\ceil#1{\lceil {#1} \rceil}

\def\Var{\ensuremath{\mathrm{\mathbf{Var}}}}
\def\E{\ensuremath{\mathrm{\mathbf{E}}}}
\def\Pr{\ensuremath{\mathrm{\mathbf{Pr}}}}

\def\AA{\mathcal{A}}

\def\PP{\mathcal{P}}

\def\UU{\mathcal{U}}

\def\MM{\mathcal{M}}

\def\accept{{\fontfamily{cmss}\selectfont accept}\xspace}
\def\reject{{\fontfamily{cmss}\selectfont reject}\xspace}

\def\col{\ensuremath{\mathrm{collisions}}}

\def\Lap{\ensuremath{\mathrm{\mathbf{Lap}}}}

\newif\ifobviousProofset
\obviousProofsettrue

\newcommand{\obviousProof}[1]{}
\ifobviousProofset
\renewcommand{\obviousProof}[1]{
  \newline
  \textcolor{blue}{
  	\textbf {Proof for our own record:} #1
  	{\hspace*{\fill}$\Box$}
  	\newline
  }
}
\fi

\makeatletter
\newcommand*{\rom}[1]{\expandafter\@slowromancap\romannumeral #1@}
\makeatother

\begin{document}
\title{Differentially Private Identity and Closeness Testing of Discrete Distributions}
\author{
Maryam Aliakbarpour\thanks{Supported by NSF grants CCF-1420692 and  CCF-1650733.} \\
CSAIL, MIT\\
\texttt{maryama@mit.edu} \\
\And
Ilias Diakonikolas\thanks{Supported by NSF Award CCF-1652862 (CAREER) and a Sloan Research Fellowship.}\\
CS, USC\\
\texttt{diakonik@usc.edu}
\And
Ronitt Rubinfeld \thanks{Supported by ISF grant 1536/14, NSF grants CCF-1420692 and  CCF-1650733.} \\
CSAIL, MIT \& TAU\\
\texttt{ronitt@csail.mit.edu} \\
}

\date{\today}
\maketitle

\begin{abstract}
We investigate the problems of identity and closeness testing over a discrete population
from random samples. Our goal is to develop efficient testers while guaranteeing
Differential Privacy to the individuals of the population. We describe an approach that 
yields sample-efficient differentially private testers for these problems.
Our theoretical results show that there exist private identity and closeness testers
that are nearly as sample-efficient as their non-private counterparts. We perform
an experimental evaluation of our algorithms on synthetic data. Our experiments
illustrate that our private testers achieve small type I and type II errors with sample size
{\em sublinear} in the domain size of the underlying distributions.
\end{abstract}

\section{Introduction} \label{sec:intro}
We consider the problem of finding sample-efficient
algorithms that allow us to
understand properties of distributions over very large discrete domains.
Such statistical tests have been traditionally studied in statistics
because of their importance in virtually every scientific endeavor that
involves data.
Recent work in the theoretical computer science community has investigated
the case when the discrete domains are large and no a priori assumptions
can be made about the distribution (for example, when it cannot be
assumed that the distribution is normal, Gaussian, or even smooth).
Optimal
methods with sublinear sample complexity
have been given for testing such properties as whether a distribution
is uniform, 
identical to a known distribution
(aka testing ``goodness-of-fit''), closeness
of two unknown distributions, and independence.

While statistical tests are very important for advancing science,
when they are performed on sensitive data representing specific individuals,
such as data  describing medical or other behavioral phenomena, 
it may be that
the outcomes of the tests reveal information that should not be divulged. 
Techniques from differential privacy give us hope that one may
obtain the scientific benefit of statistical tests without 
compromising the privacy of the individuals in the study.
Specifically, differential privacy requires 
that similar datasets have statistically close outputs -- once this
property is achieved, then provable privacy guarantees can be made.
Differential privacy is a rich and active area of study, in which
techniques have been developed 
and applied to give private algorithms for a range of data analysis tasks.

\medskip

\noindent {\bf Our Contributions} In this paper, we study the problem of hypothesis testing
in the presence of privacy constraints, focusing on the notion of differential privacy~\cite{DworkR14}.
Our emphasis is on the {\em sublinear} regime, i.e., when the number of
samples available is sublinear in the domain size of the underlying distribution(s). 
We leverage recent progress in distribution property 
testing to obtain sample-efficient {\em private} 
algorithms for the problems of testing the identity and closeness of discrete distributions.
The main conceptual message of our results is that
we can achieve differential privacy 
with only a small increase in the sample complexity 
compared to the non-private case. 
We provide sample-efficient testers for identity to a fixed distribution (goodness-of-fit)
and equivalence/closeness between two unknown distributions (both
given by samples). 
For the latter problem, we are the first to give such sample-efficient
testers with provable privacy guarantees. Our experimental evaluation on synthetic data
illustrates that our testers achieve small type I and type II errors with a sublinear number of samples
when the domain size is large.

\medskip

\noindent {\bf Technical Overview}
We now provide a brief overview of our approach.
We start by observing that there is a simple generic method
to convert a non-private tester into a 
private tester with a multiplicative overhead
in the sample complexity. This method is well-known in differential privacy,
but for the sake of completeness we describe it in Section~\ref{sec:generic}.
It is useful to contrast
the sample complexity of the generic method with the (substantially smaller)
sample complexity of our testers in Sections~\ref{sec:identity},~\ref{sec:uniformity}, and~\ref{sec:closeness}.
For convenience, throughout this paper, we work with testing algorithms
that have failure probability at most $1/3$. As we point out in Section~\ref{sec:amp}, 
this is without loss of generality: a standard amplification method shows that
we can always achieve error probability $\delta$ at the expense of a $\log(1/\delta)$
multiplicative factor in the sample complexity, even in the differentially private setting.

Our results for identity testing follow a modular
approach: First, we use a recently discovered black-box reduction of identity testing
to uniformity testing proposed by Goldreich in~\cite{Goldreich16}, building on the result of Diakonikolas and Kane~\cite{DK:16}.
We point out (Section~\ref{sec:identity}) that this reduction also applies in the private setting. 
As a corollary, we can translate any private uniformity tester to a private identity tester 
without increasing the sample size by more than a small constant factor. 
It remains to develop sample-efficient private uniformity testers. 
We develop two such private methods (Section~\ref{sec:uniformity}):
Our first method is a private version of Paninski's uniformity tester~\cite{Paninski:08}, 
which relies on the number of domain elements that appear in the sample exactly once.
This statistic has low sensitivity, allowing an easy translation to the private setting.
The sample complexity of our aforementioned uniformity tester is
$$O(\sqrt{n}/\epsilon^2 +  \sqrt n/(\epsilon \sqrt \xi)) \;,$$
where $\epsilon$ is the accuracy of the tester and $\xi$ is the privacy parameter.
As our experimental results illustrate, this private tester
performs very well in the sublinear regime.

On the other hand, it is well-known that Paninski's uniformity tester only works
when the sample size is smaller than the domain size (even in the non-private setting).
To obtain a uniformity tester that works for the entire setting of parameters, 
we develop our second method: a private version of the collisions-based tester 
first proposed by Goldreich and Ron~\cite{GR00}. The collisions-based tester was recently
shown to be sample-optimal in the non-private setting~\cite{DiakonikolasGPP16}.
The main difficulty in turning this into a private tester is that the underlying statistic (number of collisions)
has very high worst-case sensitivity. 
To overcome this issue, we add a simple pre-processing step to our tester that rejects
when there is a single element that appears many times in the sample.
(We note that a similar idea was independently used in~\cite{CaiDK17}, though the details are somewhat different.)
This allows us to reduce the effective sensitivity of our statistic and yields a sample-efficient private tester.
Specifically, the sample complexity of our collision-based private tester is
$$\tilde{O} \left(\sqrt{n}/\epsilon^2 +\sqrt {n}/(\epsilon \xi) + 1/(\epsilon^2  \xi) \right) \;.$$
For the problem of closeness testing, we build on the chi-square type optimal tester provided by Chan {\em et al.}~\cite{CDVV14}.
A major advantage of this statistic is that it has constant sensitivity. 
Hence, developing a sample-efficient private version can be achieved
by adding Laplace noise. A careful analysis shows that this noisy statistic
is still accurate without substantially increasing the sample complexity.
Specifically, the sample complexity of our private closeness tester is
$$O \left( \max \left\{ \sqrt n/\epsilon^2,  
n^{2/3}/\epsilon^{4/3}, \sqrt n / (\sqrt{\xi} \epsilon), 1/(\xi \epsilon^2) \right\} \right) \;.$$

\medskip

\noindent {\bf Related Work}
During the past two decades,  {\em distribution property testing}~\cite{BFR+:00}
-- whose roots lie in statistical hypothesis testing~\cite{NeymanP, lehmann2005testing} --
has received considerable attention by the computer science community,
see~\cite{Rub12, Canonne15} for two recent surveys.
The majority of the early work in this field has focused on characterizing the sample size needed to test properties
of arbitrary distributions of a given support size. After two decades of study, this ``worst-case''
regime is well-understood: for many properties of interest there exist
sample-optimal testers (matched by information-theoretic lower bounds)
~\cite{Paninski:08, DDSVV13, CDVV14, VV14, DKN:15, DKN:15:FOCS, ADK15, DK:16, CDGR16, 
DiakonikolasGPP16, CanonneDKS16, DKN17}.

A recent line of work~\cite{WangLK15, RogersVLG16, KiferR16, CaiDK17} has studied distribution 
testing with privacy constraints. The majority of these works~\cite{WangLK15, RogersVLG16, KiferR16}
focus on type I error analysis subject to privacy guarantees. Most relevant to ours is the very recent
work by Cai {\em et al.}~\cite{CaiDK17} that provides an identity tester with provable sample guarantees and 
bounded type I and type II errors.
Their sample upper bounds for identity testing are comparable to ours
(though quantitatively worse in some parameter settings). We also note that 
Cai {\em et al.}~\cite{CaiDK17} do not consider the more general problem of closeness testing.
Finally, recent work of Diakonikolas {\em et al.}~\cite{DiakonikolasHS15} has provided differentially private algorithms for learning
various families of discrete distributions. For the case of unstructured discrete distributions,
such algorithms inherently require sample size at least linear in the domain size, even for constant values of the 
approximation parameter.

\section {Preliminaries} \label{sec:prelims}

\noindent {\bf Notation and Basic Definitions.}
We use $[n]$ to denote the set $\{1, 2,\ldots, n\}$. 
We say $p$ is a discrete distribution over $[n]$ if $p:[n] \rightarrow [0,1]$ is a function such that $\sum_{i=1}^n p(i) = 1$, where $p(i)$ denotes the probability of element $i$ according to distribution $p$. 
For a set $S \subset [n]$, $p(S)$ denotes the total probability of the elements in $S$ (i.e., $p(S) = \sum_{i \in S} p(i)$).
For any integer $k > 0$, the $\ell^k$-norm of $p$ is equal to $\left(\sum_{i=1}^n |p(i)|^k \right)^{\frac 1 k}$, 
and it is denoted by $\|p\|_k$. The $\ell^k$-distance between two distributions $p$ and $q$ over $[n]$ 
is equal to $\left(\sum_{i=1}^n |p(i) - q(i)|^k\right)^{\frac 1 k }$. 
We use $\Lap(b)$ to denote a random variable that is drawn from a Laplace distribution with parameter $b$ and mean zero. 

The problem of {\em identity testing} (or goodness-of-fit) is the following: Given sample access
to an unknown distribution $p$ over $[n]$ and an explicit distribution $q$ over $[n]$, 
we want to distinguish, with probability at least $2/3$, between the cases that $p=q$ (completeness) and
$\|p-q\|_1 \geq \epsilon$ (soundness). The special case of this problem when $q = U_n$, the uniform distribution
over $[n]$, is called {\em uniformity testing}. The generalization of identity testing when both 
$p$ and $q$ are unknown and only accessible via samples is called {\em closeness testing.}

\noindent {\bf Differential Privacy.} 
In our context, a dataset is a multiset of samples drawn from a distribution over $[n]$.
We say that $X$ and $Y$ are {\em neighboring datasets} if they differ in exactly one element. 

\begin{definition} 
A randomized algorithm $\AA:[n]^s \rightarrow \mathcal{R}$, is
\em{$\xi$-differentially private} if for any $S \subseteq \mathcal{R}$ 
and any neighboring datasets $X, Y$, we have that
$\Pr[\AA(X) \in S] \leq e^\xi \cdot \Pr[\AA(Y) \in S] \;.$
\end{definition}
We will say that a tester is $(\epsilon, \xi)$-private, to mean that $\epsilon$ 
is the accuracy parameter, $\xi$ is the privacy parameter, and the tester outputs the right answer with probability at least $2/3$. 
For conciseness, we use the term $\xi$-private instead of $\xi$-differentially private. 
We provide more details about general techniques in differential privacy in Appendix \ref{sec:privacy_prelims}.

\section{Private Identity Testing: Reduction to Private Uniformity Testing} \label{sec:identity}
In this section, we provide a simple black-box reduction
of private identity testing (against an arbitrary fixed distribution)
to its special case of private uniformity testing.
Specifically, we point out that a recent reduction of (non-private)
identity testing to (non-private) uniformity testing works
in the private setting as well. 

In recent work~\cite{Goldreich16}, Goldreich provides a 
reduction from identity testing to uniformity testing
with only a constant multiplicative overhead in the query complexity. 
Here, we show that the same reduction works in the private setting.

Suppose we want to test identity between an unknown distribution $p$ over $[n]$ and an explicit distribution $q$.
The reduction of \cite{Goldreich16} transforms the distribution $p$ into a new distribution $p'$,
over a domain of size $O(n)$,
such that if $p=q$ then  $p'$ is the uniform distribution, and if $p$ is far from $q$, $p'$ is also far from uniform. 
Specifically, the reduction defines a randomized mapping of  a sample $i \in [n]$ from $p$
to a sample $(j, a)$ from $p'$ that depends only on the explicit distribution 
$q$. This property is crucial as it allows us to show that the transformation preserves differential privacy, as the following theorem states.

\begin{restatable}{theorem}{thmIdentity}
\label{thm:goldreich-private}
Given an $(\epsilon, \xi)$-private uniformity tester using $s(n, \epsilon, \xi)$ samples, 
there exists an $(\epsilon, \xi)$-private tester for identity using $s  = s(6n, \epsilon/3, \xi)$ samples. 
\end{restatable}
In the proof, we give a reduction from identity testing to uniformity testing. The mapping itself depends only on $q$. This fact implies that it suffices to use any private tester for uniformity. The detailed proof of the theorem is in Appendix \ref{sec:proof_identity}.

\section {Private Uniformity Testing} \label{sec:uniformity}
In this section, we provide two sample-efficient private uniformity testers.
Our testers are private versions of two well-studied (non-private) testers,
due to Goldreich and Ron~\cite{GR00} and Paninski~\cite{Paninski:08}.
Paninski's uniformity tester~\cite{Paninski:08} relies on the number of unique elements in the sample,
while~\cite{GR00} relies on the number of collisions. Both testers
are known to be sample-optimal in the non-private setting~\cite{Paninski:08, DiakonikolasGPP16}.

We give private versions of both of these algorithms.
The sample complexity of our private Paninski uniformity tester is 
$ O(\sqrt{n}/\epsilon^2 +  \sqrt n/(\epsilon \sqrt \xi)) \;.$
Therefore, as long as $\xi  = \Omega (\epsilon^2)$, the privacy requirement increases 
the sample complexity by at most a constant factor. 

Unfortunately, the aforementioned tester only succeeds when its sample size
is smaller than the domain size $n$. To be able to handle the entire range of parameters, 
we develop a private version of the collisions-based tester from~\cite{GR00}.
Our private version of the collisions tester has sample complexity 
$\tilde{O} \left(\sqrt{n}/\epsilon^2 +\sqrt {n}/(\epsilon \xi) + 1/(\epsilon^2  \xi) \right)$.
Similarly, the effect of the privacy is mild as long as $\xi  = \Omega (\epsilon)$.

\subsection{Private Uniformity Tester based on Unique Elements} \label{paninski}
Here, we provide a private tester for uniformity based on the number of unique elements. 
The number of unique elements is (negatively) related to the number of collisions  and the $\ell^2$-norm of the distribution. 
Therefore, 
the greater the number of unique elements is, 
the more the distribution appears uniform. 
To make the algorithm private, we use the Laplace mechanism which adds a small amount of noise to the number of unique elements. 
Then, we compare the number of unique elements with a threshold to decide if the distribution is uniform or far from uniform. 
The noise is big enough to make the algorithm private, 
but it does not detract much from
the accuracy of the tester. We prove this formally in Theorem \ref{thm:uniformity_paninski}.

\begin{algorithm}
\caption{The Private Algorithm for the uniformity test}\label{alg:uniformity}
\begin{algorithmic}[1]
\Procedure{Private-Uniformity-Test}{$\epsilon, \xi$}
	\State{$s \gets 5\sqrt n/(\epsilon \sqrt \xi) + 6 \sqrt n/\epsilon^2$}
	\State{$C \gets \frac {s \epsilon^2}{\sqrt n}$}
	\State{$x_1, x_2, \ldots, x_s \gets$ $s$ samples drawn from $p$}
	\State{$K \gets$ the number of unique elements in $\{x_1, x_2, \ldots, x_s\}$}
	\State{$K' \gets K + \Lap(2/\xi)$}
	\If {$K' < \E_\UU[K] - C^2/(2\epsilon^2) $}
		\State Output \reject.
	\EndIf
	\State Output \accept.
\EndProcedure
\end{algorithmic}
\end{algorithm}

\begin{restatable}{theorem}{thmUniformityPaninski}
\label{thm:uniformity_paninski}
Given $s =  O (\sqrt n/(\epsilon \sqrt \xi) + \sqrt n/\epsilon^2)$ samples from distribution $p$ over $[n]$, Algorithm \ref{alg:uniformity} is an $(\epsilon, \xi)$-private tester for uniformity if $s$ is sufficiently smaller than $n$. 
\end{restatable}
By the properties of the Laplace mechanism, we know that $K$ is private. 
Then, we show that adding noise to $K$ does not harm the accuracy of the tester because the variance of the noise is small. Using the Chebyshev inequality, we show that, with high probability,
$Z$ concentrates well around its expected value given the size of the sample set. 
Thus, since there is a large enough gap between the expected values of $K$ which case the samples came from. 
The proof of the theorem is in Appendix \ref{sec:proof_uniformity_paninski}.

\subsection{Private Uniformity Tester via Collisions} \label{ssec:collisions}
In this subsection, we describe the private version of our collisions-based uniformity tester.
The main difficulty in turning this into a private tester is that the underlying statistic (number of collisions)
has very high worst-case sensitivity. Specifically, if the sample set contains $s$ copies 
of a given domain element, by changing just one of the copies to another element, 
the number of collisions drops by an additive $s$. 
So, if we add enough noise to the statistic to cover the sensitivity of $s$, it substantially degrades 
the tester accuracy. 

To overcome this issue, we add a simple pre-processing step to our tester.
We notice that the sensitivity of the number of collisions, $f(X)$, for sample set $X$, 
depends on the maximum frequency of the element in the sample set. Let $n_i(X)$ denote 
the number of occurrences of element $i$ in the sample set $X$, 
and let $n_{\max}(X)$ denote the maximum $n_i(X)$. 
We note that for two neighboring sample sets $X$ and $Y$, 
the difference of the number of collisions, $|f(X) - f(Y)|$, is at most $n_{\max}(X)$. 
Therefore, the sensitivity of $f$ is high on $X$'s with large $n_{\max}(X)$. 
However, if the underlying distribution is uniform, we do not expect any particular element to show up very frequently. 
Hence, if $n_{\max}(X)$ is high, the algorithm can output \reject regardless of $f(X)$. 
So, the final output of the algorithm does not change drastically on $X$ and $Y$, 
while the number of collisions varies a lot. 

This simple observation forms the basis for our modified algorithm. 
The algorithm uses two statistics:  $n_{\max}$ and $f$.
If $n_{\max}$ is too large, it outputs \reject. 
Otherwise, $f(X)$ determines the output. 
In the second case, since $n_{\max}$ is not too large, $f$ has bounded sensitivity. 
Therefore, we can make it private by adding a small amount of noise to it. 
The detailed procedure is explained in Algorithm \ref{alg:uniformity_col}. 
We show correctness in Theorem \ref{thm:uniformity_col}.

\begin{algorithm}
\caption{Private uniformity tester based on the number of collisions}\label{alg:uniformity_col}
\begin{algorithmic}[1]
\Procedure{Private-Uniformity-Test}{$\epsilon, \xi$}
	\State{$s \gets 
	\Theta\left(\frac {\sqrt{n}}{\epsilon^2} + \frac{\sqrt {n \log n}}{\epsilon \, \xi^{1/2} } + \frac{\sqrt {n \max(1, \log1/\xi)}}{\epsilon \, \xi } + \frac {1}{\epsilon^2\, \xi}  \right)
	$.
	}
	\State{Let $X = \{ x_1, x_2, \ldots, x_s \}$ be a multiset of $s$ samples drawn from $p$}
	\vspace{2mm}
	\State{$n_i(X) \gets |\{j| x_j \in x \mbox{ and }x_j = i\}|$}
	\vspace{2mm}
	\State{$n_{\max}(X) \gets \max\limits_{i} n_i(X) $}	
	\State{$\hat{n}_{\max}(X) \gets n_{\max}(X) + \Lap(2/\xi)$}	
	\vspace{2mm}
	\State{$f(X) \gets$ \col $(X)$}
	\vspace{2mm}
	\State{$\eta_f \gets \max\left( \frac {3s}{2n} , 12 \, e^2 \ln24\,n \right) +  ({2\ln 12 })/ \xi + 2 \max(\ln 3, \ln 3/\xi)/\xi$}
	\vspace{2mm}
	\State{ $T\gets \max\left( \frac {3s}{2n} , 12 \, e^2 \ln24\,n \right) +  ({2\ln 12 })/ \xi$}
	\vspace{2mm}
	\State{$\hat{f}(X) \gets f(X)  + \Lap(2 \,\eta_f/\xi)$ }
		\vspace{2mm}
	\If {$\hat{n}_{\max}(X) < T  \And \hat{f}(X) < \frac{6 + \epsilon^2}{6n}{s \choose 2}$}	\vspace{2mm}
		\State{$O \gets $ \accept.}\vspace{2mm}
	\Else\vspace{2mm}
		\State{$O \gets $ \reject.}\vspace{2mm}
	\EndIf
    \State{With probability 1/6, $O \gets \{\mbox{\accept}, \mbox{\reject}\} \setminus O$.}\vspace{2mm}
    \Comment{flip the answer with probability 1/6.}
    \State{Output $O$.}
\EndProcedure
\end{algorithmic}
\end{algorithm}

\begin{restatable}{theorem}{thmUniformityCol}
 \label{thm:uniformity_col}
Algorithm \ref{alg:uniformity_col} is an $(\epsilon, \xi)$-private tester for uniformity.
\end{restatable}

To preserve the accuracy of the tester, we add a tiny amount of noise to $f$, which is not sufficient to make $\hat f$ fully private. However, we observe that the sensitivity of $f$ is high when $n_{\max}$ is high. So, the algorithm is likely to output \reject because of high $n_{\max}$ and regardless of $f$. We show that effect of $f$ on the output is small enough that the algorithm remains private. The proof of the theorem is in Appendix \ref{sec:proof_uniformity_col}. 

\section{Private Closeness Testing} \label{sec:closeness}
In this section, we give a private algorithm for testing 
closeness of two unknown discrete distributions. 
Our tester relies on the chi-squared type 
sample-optimal (non-private) closeness tester proposed in~\cite{Orlitsky:colt12} and 
analyzed in~\cite{CDVV14}. 
The closeness tester relies on the following statistic:

$$Z  \coloneqq \sum\limits_i \dfrac{(X_i - Y_i)^2 - X_i - Y_i}{X_i + Y_i} \;,$$
where $X_i$ is the number of occurrences of element $i$ in the sample set from $p$, 
and $Y_i$ is the number of occurrences of element $i$ in the sample set from $q$. 
The statistic $Z$ is chosen in a way so that its expected values in the completeness and soundness
cases differ substantially. The challenging part of the analysis involves a tight upper bound
on the variance, which allows  to show that $Z$ is well-concentrated after an appropriate number of samples.
More precisely, the following statements were shown in \cite{CDVV14}:
\begin{equation}\label{eq:EZ}
\E[Z] =  \sum\limits_i  \dfrac{\left(p(i) - q(i)\right)^2}{p(i) + q(i)}m \left( 1 - \dfrac {1 - e^{m(p(i) + q(i))}}{m(p(i) + q(i))}\right)
\geq \dfrac{m^2}{4n + 2m} \|p - q\|_1^2 \;.
\end{equation}
and
\begin{equation}\label{eq:VarZ}
\Var[Z] \leq 2\min\{m, n\} + \sum\limits_i 5m \dfrac{\left(p(i) - q(i)\right)^2}{p(i) + q(i)} \;.
\end{equation}
The private version of the above statistic is simple:
We add noise to the random variable $Z$ and work with the noisy statistic, 
denoted by $Z'$. We need to show that we still can infer the correct answer from $Z'$, 
and the noise does not incapacitate our tester. The main reason that this is indeed possible
is because the statistic $Z$ has bounded sensitivity.

Algorithm \ref{alg:private_closeness_test} is our private closeness tester 
and we prove its correctness in Theorem \ref{thm:closeness}.  

\begin{algorithm}
\caption{The private tester for closeness of two unknown distributions}\label{alg:private_closeness_test}
\begin{algorithmic}[1]
\Procedure{Private-Closeness-Test}{$\epsilon, \xi$}
	\State{$m \gets C \cdot \max \left( \dfrac{\sqrt n}{\epsilon^2}, \dfrac{n^{2/3}}{\epsilon^{4/3}}, \dfrac {\sqrt n}{\sqrt \xi \epsilon}, \dfrac 1 {\xi \epsilon^2}\right)$}
	\State{Draw $m$ samples from distributions $p$ and $q$.}
	\State{$X_i \gets$ the number of occurrences of the $i$-th element in the samples from $p$}
	\State{$Y_i \gets$ the number of occurrences of the $i$-th element in the samples from $q$}
	\State {$Z \gets \sum\limits_i \dfrac{(X_i - Y_i)^2 - X_i - Y_i}{X_i + Y_i}$ }
	\Comment{for $X_i + Y_i \not = 0$.}
	\State {$\eta \gets \Lap(8/\xi)$} 
	\vspace{2mm}
	\State {$Z' = Z + \eta$}
	\vspace{2mm}
	\State {$T \gets \dfrac {m^2 \epsilon^2}{8n + 4m}$}
	\vspace{2mm}
	\If {$Z' \leq T$ }
		\State{Accept}
	\Else 
		\State{Reject}
	\EndIf
\EndProcedure
\end{algorithmic}
\end{algorithm}

\begin{restatable}{theorem}{thmCloseness}
\label{thm:closeness}
Given sample access to two distributions $p$ and $q$, 
Algorithm \ref{alg:private_closeness_test} is an $(\epsilon, \xi)$-private tester for closeness of $p$ and $q$.  
\end{restatable}

Since the sensitivity of $Z$ is small, we can add a small amount of noise to it to make it private, using the Laplace mechanism.
Then, we show that adding the noise to $Z$ does not increase its variance drastically.
Finally, we prove by the Chebyshev inequality that, with high probability,
$Z$ concentrates well around its expected value given the size of the sample set.
Thus, since we have shown a large enough gap between the expected values of  $Z$ the in \accept and \reject cases, we can distinguish between a pair of identical distributions and a pair of distributions that are $\epsilon$-far from each other. 
The proof of the theorem is in Appendix \ref{sec:proof_closeness}.

\section{Experiments}
We provide an empirical evaluation of the proposed algorithms on synthetic data. 
All experiments were performed on a computer with a 1.6 GHz Intel(R) Core(TM ) i5-4200U CPU and 3 GB of RAM. 
The focus of the experiments is to find the minimum number of samples such that the type \rom{1} and type \rom{2} errors are small. 
In our synthetic trials, we show that for sufficiently large domain size $n$, 
our algorithms is likely to succeed with a {\em sublinear} number of samples. 
 
Specifically, given a domain size $n$, we find the (approximately) minimum number of samples 
such that the type \rom{1} and type \rom{2} errors are less than $1/3$.  
First, we pick a distribution (or a pair of distributions) that should be accepted with probability $2/3$, 
and another distribution that should be rejected with probability $2/3$. Then, we start with an initial number of samples $s$. 
For each case, we run the algorithm $r$ times on a sample set of size $s$. 
Then we estimate the accuracy of the algorithm for these sample sets. 
If the empirical accuracy is less than $2/3$ for either of the distributions, 
we increase $s$ appropriately and repeat the process until we find $s$ 
that results in an accuracy of at least $2/3$. 

\noindent {\bf Private Uniformity Testing.} 
We implemented Algorithm~\ref {alg:uniformity} to test the uniformity of a distribution in 
$\ell^1$-distance. Let $p$ be a distribution that has probability $(1+\epsilon)/n$ 
on half of the domain and probability $(1-\epsilon)/n$ on the other half. 
Clearly, $p$ is $\epsilon$-far from uniform.  
Since $p$ can be used to generate a tight sample lower bound~\cite{Paninski:08},
$p$ is in some sense the hardest instance to distinguish from the uniform distribution. 
We run the algorithm using samples from the uniform distribution and from $p$ with the following parameters: 
$\epsilon = 0.3$, $r = 300$, and $\xi = 0.2$. 
We determine the number of samples required for this tester to have accuracy 
at least $2/3$ for domain sizes $n$ ranging  
from $1$ million to 2 million (increasing $n$ by $10000$ at each step). 
The experimental results are shown in Figure~\ref{fig:uniformity}.

\begin{figure}[h] 
\centering 
\includegraphics[width=0.5\textwidth]{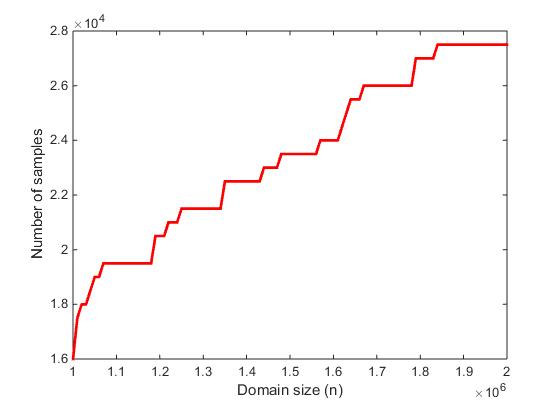}
\caption{ \label{fig:uniformity} The sample complexity of Algorithm~\ref {alg:uniformity} for private uniformity testing.}
\end{figure}

\noindent {\bf Private Identity Testing.} 
Testing uniformity is a special case of testing identity of distributions, 
and it is known to be essentially the hardest instance of the more
general problem. Similarly to~\cite{CaiDK17}, 
we consider testing identity to a distribution $q$, 
where $q$ is uniform on two disjoint subsets of the domain, 
of sizes $n_1 = n/1000$ and $n_2 = 999n/1000$. 
The total probability mass of the first subset is $0.6$ and the mass of the second one is $0.4$. 
The distribution $q$ can be viewed as a distribution which is ``heavy'' on a small
number of elements and ``light'' on the rest of the elements.
To build a distribution $p$ which is $\epsilon$-far from $q$, 
we tweak the probability of the elements in the second subset 
by $\pm \epsilon/n_2$. 
As explained in Section \ref{sec:identity}, to implement the identity test, 
we map our sample set $S$ to another sample set $S'$ on a slightly larger domain. 
Then, we use Algorithm \ref{alg:uniformity} to test the uniformity on the new domain using samples in $S'$. 
We set $\epsilon = 0.3$, $r = 200$, and $\xi = 0.2 $. 
We find the required number of  samples of this tester in order to 
have accuracy at least $2/3$, 
for $n$ from $1$ million to $2$ million 
(increasing $n$ by $10000$ at each step). 
The result is shown in Figure \ref{fig:identity}.

\begin{figure}[h] 
\centering 
\includegraphics[width=0.5\textwidth]{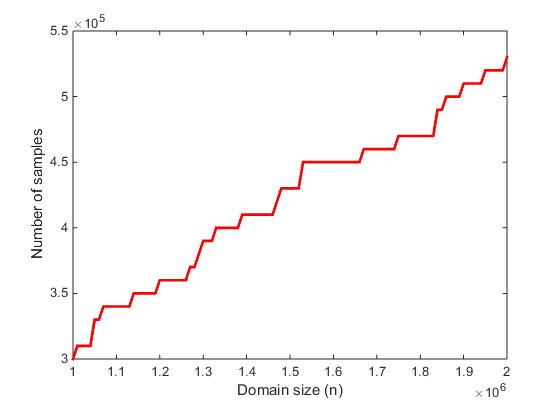}
\caption{\label{fig:identity} The sample complexity of private identity testing.}
\end{figure}

\noindent {\bf Private Closeness Testing.} 
We implemented Algorithm \ref{alg:private_closeness_test} to test closeness of two unknown distributions. 
Let $p$ be a distribution such that $n^{2/3}$ of the domain elements have 
probability $(1-\epsilon/2)/n^{2/3}$ (the ``heavy elements'') and $n/4$ ``light'' elements have 
probability $2 \epsilon/n$. 
Let $q$ be a distribution that has probability $(1-\epsilon/2)/n^{2/3}$ 
on the same set of heavy elements as $p$, 
and for a disjoint set of $n/4$ light elements assigns
probability $2 \epsilon/n$. Since the light elements are disjoint,
it is clear that $p$ is $\epsilon$-far from $q$. 
It has been shown in \cite{Batu13} and \cite{CDVV14}, 
that this pair of distributions yields a family of pairs of distributions
(via randomly permuting the names of the elements) which can be used to
give a tight lower bound on the sample complexity for the problem of
testing closeness.

To evaluate the accuracy of our algorithm, we use the tester to distinguish the following pairs: 
$(q, q)$ and $(p, q)$. 
We set $\epsilon = 0.3$, $r = 200$, and $\xi = 0.2 $. 
We find the required number of  samples of this tester in order to 
have accuracy at least $2/3$, for $n$ raging from $1$ million to $2$ million
(increasing $n$ by $10000$ at each step). The result is shown in Figure \ref{fig:closeness}.

\begin{figure}[h] 
\centering 
\includegraphics[width=0.5\textwidth]{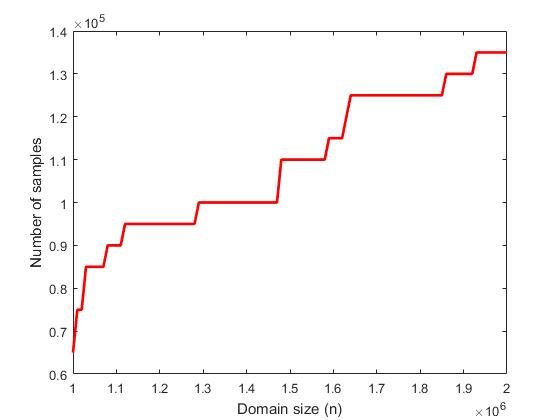}
\caption{\label{fig:closeness} The sample complexity of private closeness testing.}
\end{figure}

\bibliographystyle{plain}
\bibliography{allrefs}

\newpage

\appendix
\section {General Techniques in Differential Privacy} \label{sec:privacy_prelims}
A standard mechanism in the privacy literature, the {\em Laplace mechanism}, 
perturbs the output of an algorithm
by adding Laplace noise to make the output private. 
Assume the algorithm computes a function $f:[n]^s \rightarrow \mathbb{R}$. 
The amount of noise required
depends on the privacy parameter, $\xi$, 
and how much $f$ varies over two neighboring datasets. More precisely, 
this variation of $f$ is called {\em sensitivity} 
of the function and it is defined as: 

$$\Delta f = \max\limits_{\mbox{\scriptsize neighboring } x, y} |f(x) - f(y)| \;.$$
The noise is drawn from a Laplace distribution with parameter $b = \Delta f/\xi$. 
We denote the noise by $\Lap(b)$.  More precisely, 

$$\Pr[\Lap(b) = x] = \frac 1 {2b} \exp\left(-\frac{|x|}b\right).$$
The following is well-known:
\begin{lemma}[The Laplace mechanism (Theorem 3.6 in \cite{DworkR14})] \label{lem:Laplace_mechanism}
Assume there is an algorithm $\AA$ 
that on input $x$, outputs $f(x) + \Lap(\Delta f/\xi)$.  
Then $\AA$ is $\xi$-private. 
\end{lemma}
Note that the expected value of $\Lap(b)$ is zero. 
Therefore, the expected value of the output remains $\E[f(x)]$. 
Since we draw the noise independently from $x$, 
the variance of the output is increased by $\Var[\Lap(b)] = 2b^2$. 

Moreover, the following lemmas help us 
understand how the privacy guarantee changes 
if we process the output of one or more private algorithm. 
\begin{lemma}[Post-processing  (Proposition 2.1 in \cite{DworkR14})] \label{lem:post_processing} 
Assume $\AA$ is a $\xi$-private algorithm. 
Any algorithm that on input $x$ outputs a function $f(\AA(x))$ is also $\xi$-private. 
\end{lemma} \label{lem:composition}

\begin{lemma}[Composition Theorem (Theorem 3.16 in \cite{DworkR14})] \label{lem:composition}
Let $\AA_i:[n]^s \rightarrow \mathbb R$ be a $\xi_i$-private algorithm for $i = 1, \ldots, k$. 
Any algorithm that on input $x$ outputs a function $f\left(\AA_1(x), \AA_2(x), \ldots, \AA_k(x)\right)$ is 
$\left(\sum_{i=1}^k \xi_i\right)$-private. 
\end{lemma}

\section{Generic Differentially Private Tester} \label{sec:generic}



In this section, we describe a simple generic method
to convert a non-private tester into a 
private tester with a multiplicative overhead
in the sample complexity. While this method is known 
in the differential privacy community, it is useful to contrast
its sample complexity with the (substantially smaller)
sample complexity of our testers in Sections~\ref{sec:identity},~\ref{sec:uniformity}, and~\ref{sec:closeness}.

Assume $\AA$ is a tester that draws $s(n, \epsilon)$ samples. 
The idea is to draw $m \cdot s(n, \epsilon)$ samples 
for a sufficiently large $m$, and from this sample,
to pick a random subset of size $s(n, \epsilon)$ samples. 
Then, the new tester runs $\AA$ on the randomly chosen subset and outputs
$\AA$'s output.
Given two sample sets that differ in one sample, the new private tester
will give the same output whenever a chunk that does not contain the differing
sample is chosen, which happens with probability at most $1/m$.   
This reduction to a non-private tester is described in Algorithm \ref{alg:general}. 
We formally show its correctness in Theorem~\ref{thm:general}. 

\begin{algorithm}
\caption{Reduction to a non-private tester}\label{alg:general}
\begin{algorithmic}[1]
\Procedure{General-Private-Tester}{$\epsilon, \xi$}
	\State{$m \gets \lceil \frac{6}{ \xi }\rceil$}
	\State{$s' \gets m \cdot s(n, \epsilon)$}
    \State{$x_1, x_2, \ldots, x_{s'} \gets $ $s'$ samples from $p$.}
    \State{$r \gets$ Pick a random number from $[m]$.}
    \State{$O \gets \AA\left(\{x_{(r-1)s + 1}, x_{(r-1)s + 2}, \ldots, x_{rs}\}\right)$.}
    \State{With probability 1/6, $O \gets \{\mbox{\accept}, \mbox{\reject}\} \setminus O$.}
    \Comment{flip the answer with probability 1/6.}
    \State{Output $O$.}
\EndProcedure
\end{algorithmic}
\end{algorithm}

\begin{theorem}\label{thm:general}
Let $\AA$ be an $\epsilon$-tester for property $\PP$ 
that uses $s(n, \epsilon)$ samples from distribution $p$ over $[n]$. 
Algorithm \ref{alg:general} is an $(\epsilon, \xi)$-private 
property tester  for property $\PP$ using
$O\left (s(n, \epsilon)  / \xi \right)$ samples.
\end{theorem} 

\begin{proof}
Suppose $\AA$ is an $\epsilon$-tester for property $\PP$ 
that uses $s(n, \epsilon)$ samples. Without loss of generality, 
assume the tester $\AA$ errs with 
probability at most $1/6$\footnote{This can be achieved by the standard amplification method 
(i.e., running the tester $O(1)$ times and taking the majority answer). 
The new sample complexity grows by at most a constant multiplicative factor.}.
Since the output of $\AA$ is then 
flipped with probability $1/6$, by the union bound, 
the probability that Algorithm \ref{alg:general} 
errs is at most $1/3$, and it is thus an  $\epsilon$-tester for uniformity. 

To prove the privacy guarantee, 
let $m$ be $\ceil{6/\xi}$, 
and let $X = \{x_1, x_2, \ldots, x_{s'}\}$ and 
$Y = \{y_1, y_2, \ldots, y_{s'}\}$ be two 
sample sets of size $s' \coloneqq m \cdot s(n, \epsilon)$ 
that differ in exactly one sample. 
Without loss of generality, we assume they differ in the first sample: 
$x_i = y_i$ for $i>1$ and $x_1 \neq y_1$. 
Algorithm \ref{alg:general} picks a random 
number, $r$, in $[m]$ and feeds $\AA$ with the $r$-th chunk of 
size $s(n, \epsilon)$ from the input sample set. 
If $r \neq 1$, 
the distribution of the output is identical $X$ and $Y$. 
Let $T(X)$ indicate the output of Algorithm \ref{alg:general} on input $X$. 
More precisely, we have 

$$
\begin{array}{ll}
\Pr [T(X) = \mbox{\reject}] & =  \sum\limits_{i=1}^m  \Pr[T(X) = \mbox{\reject} | r = i] \cdot \Pr[r = i] 
\vspace{2mm}\\ 
& = \frac 1 m  \sum\limits_{i=1}^m  \Pr[T(X) = \mbox{\reject} | r = i] 
\vspace{2mm}\\ 
& = \frac 1 m \sum\limits_{i=2}^m  \Pr[T(Y) = \mbox{\reject} | r = i] + \frac 1 m\,\Pr[T(X) = \mbox{\reject} | r = 1]
\vspace{2mm}\\ 
& \leq \frac 1 m  \sum\limits_{i=1}^m  \Pr[T(Y) = \mbox{\reject} | r = i] 
+ \frac 1 m 
\vspace{2mm}\\ 
& \leq \Pr [T(Y) = \mbox{\reject}] + \frac 1 m \;.
\end{array}
$$
Since we change the output of $\AA$ with probability $1/6$, 
it is not hard to see that $\Pr[T(Y) = \mbox{\reject}]$ is at least $1/6$ for any input $y$. Thus, 

$$\dfrac {\Pr [T(X) = \mbox{\reject}] }{\Pr [T(Y) = \mbox{\reject}] } \leq 1 + \dfrac 6 { m} \leq 1 + \xi < e^{\xi} \;.$$
Similarly, we can show the above inequality when the output is \accept. 
Thus, the algorithm is $\xi$-private. 
\end{proof}

\section{Amplification of Confidence Parameter in the Private Setting} \label{sec:amp}

For convenience, throughout this paper we work with testing algorithms
that have failure probability at most $1/3$. Here we point out that this is without 
loss of generality, since a standard
amplification method also succeeds in the differentially private setting.

\begin{algorithm}
\caption{Amplified confidence parameter}\label{alg:amplification}
\begin{algorithmic}[1]
\Procedure{Amplifier}{$n, \epsilon, \xi$}
	\State{$m \gets 18 \lceil \ln\frac{1}{\delta}\rceil + 1$}
	\State{$s \gets s(n, \epsilon, \xi)$}
	\State{$ c \gets 0$} 
	\For {$i = 1, \ldots, m$}
    		\State{$X^{(i)} \gets $ a set of $s$ samples from $p$ }
		\State{Run $\AA$ using samples in $X^{(i)}$.}
		\If{$\AA$ accepts}
			\State{$c \gets c+1$}
		\EndIf
	\EndFor
	\If {$c \geq  m/2$}
		\State{Output \accept.}
	\Else
		\State{Output \reject.}
	\EndIf
\EndProcedure
\end{algorithmic}
\end{algorithm}

\begin{theorem}
Given $\AA$, an $(\epsilon, \xi)$-private tester for property $\PP$, 
such that $\AA$ uses $s(n, \epsilon, \xi)$ samples for any 
input distribution $p$ over $[n]$. 
Algorithm \ref{alg:amplification} 
is an $(\epsilon, \xi)$-private tester for 
property $\PP$, using $O\left (\log 1/\delta \cdot s(n, \epsilon, \xi)\right)$ 
samples from $p$,
that outputs the correct answer with probability $1 - \delta$.
\end{theorem} 
\begin{proof}
First, we show that algorithm 
\ref{alg:amplification}
is $\xi$-private: 
Let $X$ and $Y$ be two sample sets of size $m\cdot s$ (where
$m$ and $s$ are as defined in 
algorithm \ref{alg:amplification})
 that differ only in one sample. Without loss of generality, assume they differ in the first sample. 
Therefore, $X^{(1)}$ and $Y^{(1)}$ differ in only one sample, 
and for $i > 1$, $X^{(i)}$ and $Y^{(i)}$ are identical. 
Hence, the distribution of the output of $\AA$ in all of the iterations 
except the first one is identical for both $X$ and $Y$. 
For the first iteration, the distribution over the output of $\AA$ cannot change drastically, 
because $\AA$ is a $\xi$-private algorithm. 
More formally, we have the following:

$$\Pr[\AA(X^{(i)}) = \mbox{\accept}] = \Pr[\AA(Y^{(i)}) = \mbox{\accept}]\quad \quad \quad \mbox{for } i > 1,$$
and 

$$\Pr[\AA(X^{(1)}) = \mbox{\accept}] \leq e^\xi \cdot \Pr[\AA(Y^{(1)}) = \mbox{\accept}].$$
An analogous argument holds when the output is \reject. 
Let $T(X)$ indicate the output of Algorithm \ref{alg:general} on input $X$.
 Let $\sigma(X^{(i)})$ be an indicator variable that is one if $\AA$ outputs \accept on input $X^{(i)}$ and zero otherwise.
Since iterations of the algorithm are independent, we have:

$$
\begin{array}{ll}
\Pr[T(X) = \mbox{\accept}] & = \Pr\left[\sum\limits_{i=1}^m \sigma(X^{(i)}) \geq 9 \ceil{ \ln \frac{1}{\delta}} + 1 \right]
\vspace{2mm} \\
& = \Pr\left[\sigma(X^{(1)}) = 1\right] \cdot \Pr\left[\sum\limits_{i=2}^m \sigma(X^{(i)}) =  9 \ceil{ \ln \frac{1}{\delta}}\right] + \Pr\left[\sum\limits_{i=2}^m \sigma(X^{(i)}) \geq  9 \ceil{ \ln \frac{1}{\delta}} + 1 \right]
\vspace{2mm} \\
& \leq e^\xi \cdot \Pr\left[\sigma(Y^{(1)}) = 1\right] \cdot \Pr\left[\sum\limits_{i=2}^m \sigma(Y^{(i)}) =  9 \ceil{ \ln \frac{1}{\delta}} \right] + \Pr\left[\sum\limits_{i=2}^m \sigma(Y^{(i)}) \geq 9 \ceil{ \ln \frac{1}{\delta}} + 1 \right]
\vspace{2mm} \\
& \leq e^\xi \cdot \Pr\left[\sum\limits_{i=1}^m \sigma(Y^{(i)}) \geq9 \ceil{ \ln \frac{1}{\delta}} + 1 \right]
\vspace{2mm} \\
& \leq e^\xi \cdot \Pr[T(Y) = \mbox{\accept}] \;.
\end{array}$$
An analogous inequality holds  
for the case where the output is \reject. 
Therefore, Algorithm \ref{alg:amplification} is $\xi$-private. 
Moreover, the output of the algorithm is wrong 
only if the majority of the invocations of $\AA$ 
return the wrong answer (i.e. more than $9 \ceil{ \ln {1}/{\delta}}$ times).
However, $\AA$ errs with probability at most $1/3$ by definition. 
By the Hoeffding bound,  the probability of outputting the wrong answer is

$$
\Pr\left[T(X) \mbox{ is wrong}\right] \leq e^{-2m/36} \leq \delta \;.$$

Thus, the total error probability is at most $\delta$. 
Therefore, Algorithm \ref{alg:amplification} is an $(\epsilon, \xi)$-private tester that outputs the correct answer with probability $1-\delta$. 
\end{proof}

\section{Proof of Theorem \ref {thm:goldreich-private}} \label{sec:proof_identity}
\thmIdentity*

\begin{proof}
Given $s$ samples from $p$, we map them to $s$ samples from $p'$ using the following mapping: 
\begin{enumerate}
\item Given sample $i$ from $p$, the process $F_1(i)$ flips a fair coin.
If the coin is Heads, $F_1(i)$ outputs $i$, otherwise, $F_1(i)$ outputs
$j$ drawn uniformly from $[n]$.  
Let $p_1$ denote the output
distribution of $F_1(i)$'s.   It is clear that $p_1(i) = (1/2) p(i) + 1/(2n).$
We define $q_1(i)$ similarly. 
\item  Let $m_i = \lfloor 3 n (q(i) + 1/n)\rfloor$. 
Given $j$ and the output of process $F_1(i)$ where $i$ is drawn from $p$, 
process $F_2(i)$ outputs $j$ with probability $m_i / \left(3 n (q(i) + 1/n)\right)$ and $n+1$ otherwise. 
Let $p_2$ denote the output distribution of the $F_2(i)$'s.   It is not hard to see that 

$$p_2(j) = p_1(j) \cdot \frac {m_j }{3 n (q(j) + 1/n)} = \frac 1 2 \cdot \left(p(j) + \frac 1 n\right) \cdot \frac {m_j }{3 n (q(j) + 1/n)}$$
for all $i \in [n]$, and $p_2(n+1) = 1 - \sum_{\ell=1}^n p_2(\ell)$.
We define $q_2(i)$ similarly. 
\item Given $k$, the output of process $F_2(i)$ where $i$ is drawn from $p$, 
we output  $F_3(i) = (k, a)$ such that $a$ is uniformly chosen
from $[6n q_2(k)]$. 
Note that for $k \in [n]$, $6n q_2(k)$ is equal to $m_k$ and it is an integer, 
so the set $[6n q_2(k)]$ is well-defined. 
We denote the distribution of $F_3(k)$'s as $p'$. It is not hard to see that if $p = q$, then

$$p'\left((k, a)\right) = \frac 1 2 \cdot \left(q(j) + \frac 1 n\right) \cdot \frac {m_j }{3 n (q(j) + 1/n)} \cdot \frac 1 {m_j} = \frac 1 {(6n)}$$
for $j \in [n]$. For $k = n+1$, 

$$6n \, q_2(n+1) = 6n -  6n\sum\limits_{\ell=1}^n \frac {m_{\ell}}{6n} = 6n - \sum\limits_{\ell=1}^n  {m_{\ell}}.$$ 
is also an integer. Therefore, $p'((n+1, a))$ is also $q_2(n+1)/(6n\,q_2(n+1)) = 1/6n$. 
\end{enumerate}
Thus, if $p = q$, then $p'$ will be a uniform distribution. Similarly, if $\|p-q\|_1 \geq \epsilon$ then 
$\|p'-U\|_1 \geq \epsilon/3$.
For a detailed proof, see \cite{Goldreich16}.

Then, we run the private uniformity tester using the samples from  $p'$, 
and output the answer of the tester. As shown in \cite{Goldreich16}, if 
$p$ is $\epsilon$-far from $q$, then $p'$ is $\epsilon/3$-far from uniform; 
and if $p$ is identical to $q$, then $p'$ is uniform. 
Therefore, the algorithm is an $\epsilon$-tester for identity. 
It suffices to show that the algorithm preserves differential privacy. 

Assume $X$ is the set of samples drawn from $p$, 
and denote by $\pi$ the bits of randomness that the mapping used 
to build $X'_\pi$, the set of samples from $p'$. Assume $Y$ is a sample set from $p$ 
that differs from $X$ in exactly one location. Then $X'_\pi$ 
also differs from $X'_\pi$  in at most one location, 
because each sample from $p$ is used in generating 
exactly one sample from $p'$. Let $\AA$ be the $(\epsilon, \xi)$-private uniformity 
tester and denote by $\AA(X'_\pi)$ the output of the tester on input $X'_\pi$. 
Since the algorithm is $\xi$-private, we have:

$$\Pr[\AA(S'_\pi) = \mbox{\accept}] \leq e^\xi \cdot \Pr[\AA(Y) = \mbox{\accept}] \;.$$
Let $T(X)$ denote the output of our algorithm. 
By construction, we have

\begin{eqnarray*}
\dfrac {\Pr[T(X) = \mbox{\accept}]} {\Pr[T(Y) = \mbox{\accept}]} 
&=& \dfrac{\sum_{\pi} \Pr[\AA( X'_\pi)  = \mbox{\accept} ] \cdot \Pr[\pi]}{\sum_{\pi} \Pr[ \AA(Y'_\pi) = \mbox{\accept}] \cdot \Pr[\pi]} \\
&\leq& \dfrac{\sum_{\pi} e^{\xi} \cdot \Pr[\AA( Y'_\pi)  = \mbox{\accept} ] \cdot \Pr[\pi]}{\sum_{\pi} \Pr[ \AA(Y'_\pi) = \mbox{\accept}] \cdot \Pr[\pi]}
\leq e^\xi \;.
\end{eqnarray*}
By the same argument, we can show the above inequality holds 
when the output is \reject. 
Therefore, our algorithm is an $(\epsilon, \xi)$-private tester. 
\end{proof}

\section{Proof of Theorem \ref{thm:uniformity_paninski}} \label{sec:proof_uniformity_paninski}

\thmUniformityPaninski*

\begin{proof}
Algorithm \ref{alg:uniformity} draws $s$ samples from the underlying distribution $\PP$. 
We use the Laplace mechanism to make the algorithm private: 
Let $K$ be the number of unique elements in the sample set. 
Since changing one sample in the sample set can change
the number of unique elements by no more than two, 
adding Laplace noise with parameter $2/\xi$ to $K$ makes it $\xi$-private. 
Using the composition theorem \ref{lem:composition}, 
the algorithm is $\xi$-private.   

To show the algorithm is an $\epsilon$-tester, 
we prove the statistic $K'$ concentrates well around its expected 
value in both the soundness and completeness cases. 
Using Lemmas 1 and  2 in \cite{Paninski:08}, 
we have the following inequalities for the number of unique elements:
\begin{equation}\label{eq:pan_dis}
\E_\UU[K] - \E_\PP[K] \geq \frac {s^2 \|p - U_n\|_1^2}{n}
\end{equation}
and 
\begin{equation}\label{eq:pan_var}
\Var[K] \leq \E_\UU[K] - E_\PP[K] + \dfrac {s^2}{n}.
\end{equation}

First, we show the algorithm is an
$\epsilon$-tester for uniformity. 
Then, we prove that it is $\xi$-private. 

Assume that the underlying distribution is the uniform distribution. 
Note that $\E[K] = \E[K']$. Then, by the Chebyshev inequality and Equation \ref{eq:pan_var} we have that:
$$
\begin{array}{ll}
\Pr[|K' - \E_\UU[K]| \geq \dfrac{C^2}{2\epsilon^2}] & = \Pr[|K' - \E_\UU[K']| \geq \dfrac{C^2}{2\epsilon^2}]
\vspace{2mm} \\ 
 & \leq \dfrac{4\epsilon^4  }{C^4}\ \Var[K']\vspace{2mm} \\ 

& \leq \dfrac{4\epsilon^4  }{C^4}\ (\Var [K] + \Var[\Lap(2/\xi)])\vspace{2mm} \\ 
& \leq \dfrac{4\epsilon^4  }{C^4}\ (\dfrac{s^2} n + \dfrac 8 {\xi^2})\vspace{2mm} \\ 
& \leq \dfrac 4{C^2} + \dfrac {32 \epsilon^4}{C^4 \xi^2}  \vspace{2mm} \\
& \leq \frac1 3 \;,
\end{array}
$$
where the last inequality comes from the fact that $C \geq \max (3.73 \ \epsilon /\sqrt{\xi}, 4.9)$. 
Thus, the probability of rejecting $\PP$ is less than $1/3$. 

Now suppose $\PP$ is a distribution which is $\epsilon$-far from uniform. 
Again by the Chebyshev inequality and Equation (\ref{eq:pan_var}) we have that:
$$
\begin{array}{ll}
\Pr[|K' - \E_\PP[K]|  \geq (\E_\UU[K] - \E_\PP[K])/2] & = \Pr[|K' - \E_\PP[K']|  \geq (\E_\UU[K] - \E_\PP[K])/2] 
\vspace{2mm} \\ 
& \leq \dfrac{4 \Var[K']}{(\E_\UU[K] - \E_\PP[K])^2}
\vspace{2mm}\\ 
& = \dfrac{4 (\Var[K] + \Var[\Lap(2/\xi)])}{(\E_\UU[K] - \E_\PP[K])^2}
\vspace{2mm}\\ 
& = \dfrac{4 (\Var[K] + 8 \xi^{-2})}{(\E_\UU[K] - \E_\PP[K])^2}
\vspace{2mm}\\ 
& \leq \dfrac{4 ( \E_\UU[K] - \E_\PP[K] + s^2/n + 8 \xi^{-2})}{(\E_\UU[K] - \E_\PP[K])^2}
\vspace{2mm}\\ 
& \leq \dfrac{ 4 }{ \E_\UU[K] - \E_\PP[K]} + \dfrac{4 s^2/n + 32/\xi^2}{(\E_\UU[K] - \E_\PP[K])^2} 
\vspace{2mm} \;.
\end{array}
$$
On the other hand by Equation \ref{eq:pan_dis}, $\E_\UU[K] - \E_\PP[K]$ is at least $C^2 /\epsilon^2$. Thus,
$$
\begin{array}{ll}
\Pr[|K' - \E_\PP[K]|  \geq (\E_\UU[K] - \E_\PP[K])/2] 
& \leq \dfrac{ 4 \epsilon^2}{C^2} + \dfrac {4s^2 \epsilon^4}{C^4n} + \dfrac{32 \epsilon^4}{C^4 \xi^2} 
\vspace{2mm}\\ 
& \leq \frac 1 3\vspace{2mm} \;,
\end{array}
$$
where the last inequality is true when $C \geq \max \left(6\epsilon, 6, 4.12 \epsilon/\sqrt \xi \right)$. 
Thus, the probability of accepting is less than $1/3$.
\end{proof}

\section{Proof of Theorem \ref{thm:uniformity_col}} \label{sec:proof_uniformity_col}

\thmUniformityCol*

\begin{proof}
Let $X = \{x_1, \ldots, x_s\}$ be a set of $s$ samples from $p$. Let $f(X)$ be the number of collisions in $X$. 
All variables are as defined in Algorithm~\ref{alg:uniformity_col}. First, we show that $\hat f(X)$ 
and $n_{\max}(X)$ concentrate well around their expected values. 

\begin{restatable}{lemma}{accurateHatF}
\label{lem:accurate_hat_f}

If $s$ is $\Theta\left(\frac {\sqrt{n}}{\epsilon^2} + \frac{\sqrt {n \log n}}{\epsilon \, \xi^{1/2} } + \frac{\sqrt {n \max(1, \log1/\xi)}}{\epsilon \, \xi } + \frac {1}{\epsilon^2\, \xi}  \right)$, the following holds with probability at least $11/12$:
\begin{itemize}
\item If $p$ is the uniform distribution, then $\hat f(X)$ is less than $\frac{1 + \epsilon^2/6}{n}{s \choose 2}$.
\item If $p$ is $\epsilon$-far from uniform, then $\hat f(X)$ is  greater than $\frac{1 + \epsilon^2/6}{n}{s \choose 2}$.
\end{itemize}
\end{restatable}

\begin{proof} 
First, we compute the expected value of $\hat f(X)$. Since the expected value of the noise is zero, $\E[\hat f(X)]$ is equal to $\E[f(X)]$. So, if $p$ is uniform, then $\E[\hat f(X)]$ is ${s \choose 2}/n$, and if $p$ is $\epsilon$-far from uniform $\E[\hat f(X)]$ is at least $(1 + \epsilon^2) {s \choose 2}/n$. Let $\alpha$ satisfy $\|p\|_2^2 =  (1 + \alpha) /n$ and $\sigma$ be the standard deviation of $\hat f(X)$. We make an assumption that  ${ |\epsilon^2/6 - \alpha|} {s \choose 2}/n$ is at least $\sqrt{12} \sigma$. Below, this assumption concludes  the statement of the lemma. Later, we prove that the assumption holds for sufficiently large $s$. 

The conditions of the lemma hold if $\hat f(X)$ is closer to its expected value than the distance of the threshold, $\frac{1 + \epsilon^2/6}{n}{s \choose 2}$,  to its expected value. Using the Chebyshev inequality, the probability that the conditions do not hold is at most

$$
\begin{array}{ll}
\Pr\left [ \left|\hat f(X) -\E[f(X)]\right| > \left|\frac{1 + \epsilon^2/6} n {s \choose 2}- \E[f(X)] \right| \right ] 
& = \Pr\left [ \left|\hat f(X) -\E[f(X)]\right| > \frac{ |\epsilon^2/6 - \alpha|} n {s \choose 2}] \right ]
\vspace{2mm}\\
& \leq \Pr\left[ |\hat f(X) - \E[f(x)]| \geq \sqrt {12} \sigma \right] \leq \frac 1 {12} \;.
\end{array}
$$
Thus, it is sufficient to show that 
\begin{equation}\label{eq:sigma_bound}
\frac{ |\epsilon^2/6 - \alpha|} n {s \choose 2} \geq \sqrt{12}\sigma.
\end{equation}
Recall that $\sigma^2$ is equal to $\Var[f(X)] + \Var[\Lap(2\eta_f/\xi)]$, so $\sigma$ is at most $\sqrt {2\max(\Var[f(x)], 8 \eta_f^2/\xi^2)}$. Hence, we prove two stronger inequalities that yield to Equation  (\ref{eq:sigma_bound}):
\begin{equation}\label{eq:s_var_f}
s \geq \sqrt{\dfrac{20 \, n \, \sqrt{\Var[f(X)]} }{|\epsilon^2/6 - \alpha|}}, 
\end{equation}
and 
\begin{equation} \label{eq:s_var_noise}
s \geq \sqrt{\dfrac{28 \, n \, \eta_f }{\xi |\epsilon^2/6 - \alpha|}} \;.
\end{equation}
Using a similar proof to the proof of  Lemma 4 in \cite{DiakonikolasGPP16}, the inequality of Equation (\ref{eq:s_var_f}) holds for $s = c \sqrt n / \epsilon^2$ for sufficiently large constant $c$. 
Now, we focus on Equation (\ref{eq:s_var_noise}). If $p$ is a uniform distribution, $\alpha$ is zero, and if $p$ is $\epsilon$-far form being uniform, then $\alpha$ is at least $\epsilon^2$. Therefore, the denominator is at least $\epsilon^2/6$.  Solving Equation (\ref{eq:s_var_noise}) for $s$, we have:
$$s \geq c' \cdot \left(\frac {1}{\epsilon^2\, \xi} + \frac{\sqrt {n \log n}}{\epsilon \, \xi^{1/2} } 
+ \frac{\sqrt {n \max(1, \log1/\xi)}}{\epsilon \, \xi } \right) \;.$$
Hence, for sufficiently large constant $c'$, 
Equation (\ref{eq:sigma_bound}) holds and the proof is complete. 
\end{proof}

We have the following lemma:

\begin{lemma} \label{lem:accurate_hat_n} 
Let $X$ be a sample set of size $s$ from the uniform distribution over $[n]$.  With probability $11/12$, we have
$$\hat{n}_{\max} \leq \max\left(\frac 3 2 \cdot \frac s n , 12 \, e^2 \ln 24 n\right) + \frac {2\ln 12} \xi  \;.$$
\end{lemma}
\begin{proof}
First, we show that $n_{\max}(X)$ is at most $\max\left(3s/(2n) , 12 \, e^2 \ln 24n\right)$ with probability at least 23/24. 
It suffices to show that all of the $n_i(X)$'s are smaller than this bound. 
Consider the following cases: First, assume $s$ is at most $12\, n \cdot \ln (24n)$. 
Let $k \coloneqq 12 \, e^2 \cdot \ln (24n) \geq e^2 s/n$. If $s \leq k$, then 
$n_{\max} (X)$ is at most $\max \left( 3s/(2n), 12 e^2 \ln 24 n \right)$.
Otherwise, 
$$\Pr\left[n_i(X) > k \right] \leq {s \choose k} \cdot \frac 1 {n^k} \leq \left(\frac{s\cdot e}{k}\right)^k \cdot \frac 1 {n^k} \leq e^{-k} \leq \frac 1 {24\,n}.
$$
Second, assume $s$ is greater than $12\, n \cdot \ln (24n)$. By the Chernoff bound,  we have
$$\Pr\left[n_i(X) > \frac s n \left(1 + \frac 12\right) \right] \leq \exp(-\frac s {12\,n}) \leq \frac 1 {24\,n} \;.
$$
Thus, 
$$\Pr\left[n_i(X) > \max \left(3s/(2n), 12\, e^2 \ln 24 n \right) \right] \leq \frac 1 {24\, n} \;.$$
Using the union bound, with probability $23/24$ all the $n_i(X)$'s, 
and consequently $n_{\max}(X)$, are smaller than $\max \left(3s/(2n), 12\, e^2 \ln 24 n \right)$.

Moreover, based on the properties of the Laplace distribution, we have
$$ \Pr \left[ \Lap(2/\xi) \geq \frac {2\ln 12} \xi   \right] \leq \frac {\exp(-\ln12) }2 \leq \frac 1 {24} \;.
$$
By the union bound, $n_{\max}(X)$ and $\Lap(2/\xi)$ are not exceeding the aforementioned bounds with probability $11/12$. 
Therefore, we have 
$$\Pr\left[ \hat{n}_{\max} < \max \left(3s/(2n), 12\, e^2 \ln 24 n \right) + \frac {2\ln 12} \xi\right] \geq \frac {11} {12} \;.$$
Thus, the proof is complete. 
\end{proof}

Given $X$, we define two probabilistic events, $E_1(X)$ and $E_2(X)$, to be 	
$$
E_1(X)\colon \hat{n}_{\max} < T \quad \quad E_2(X)\colon \hat{f}(X) <\dfrac{6 + \epsilon^2}{6n}{s \choose 2} \;,
$$
where the probability is taken over the randomness of the noise. Observe that $E_1(X)$ and $E_2(X)$ are independent. 
We use $\overline{E_1}(X)$ and $\overline{E_2}(X)$ to indicate the complementary events. 
Let $\MM(X)$ denote the output of the algorithm when the input sample set is $X$.  We set the output, $O$, to \accept, if both $E_1(X)$ and $E_2(X)$ are true, and at the end of the algorithm we may flip the output with small probability. 
Here, we prove the probability of outputting the correct answer is at least $2/3$. Consider two following cases:

\textbf{({\textit{i}}) $p$ is uniform:} Using Lemma \ref{lem:accurate_hat_n}, with probability at least $11/12$ we have that
$\hat n_{\max}$ is less than $T$. By lemma \ref{lem:accurate_hat_f}, 
$\hat{f}$ is less than ${s \choose 2}|1 + \epsilon^2/6|/n$ with probability at least $11/12$. 
Therefore, $\Pr[\overline{E_1}(X)]$ and $\Pr[\overline{E_2}(X)]$ are at most $1/12$. 
At the end of the algorithm, we flip the output with probability at most $1/6$. 
Using the union bound, we have
$$\Pr\left[\MM(X) = \mbox{\accept}\right] \geq 1 - \Pr\left[\overline{E_1}(X) \right] - \Pr\left[\overline{E_2}(X) \right]- \frac 1 6 \geq \frac 2 3 \;.$$

{\textbf{({\textit{ii}}) $p$ is $\epsilon$-far from uniform:}}  By lemma \ref{lem:accurate_hat_f}, 
$\hat{f}(X)$ is greater than  ${s \choose 2}|1 + \epsilon^2/6|/n$ with probability at least $11/12$, so $\Pr[E_2(X)]$ is at most $1/12$. We flip the output of the algorithm with probability at most $1/6$.  As a result, we have
$$\Pr\left[\MM(X) = \mbox{\reject}\right] \geq 1  - \Pr\left[E_2(X)\right]- \frac 1 6 \geq \frac 2 3  \;.$$
Thus, with probability at least $2/3$ we output the correct answer. 

%

In the rest of the proof, we focus on proving the privacy guarantee. It is not hard to see that $|n_{\max}(X) - n_{\max}(Y)|$ is at most one. By the properties of the Laplace mechanism in Lemma \ref{lem:Laplace_mechanism}, $\hat{n}_{\max}(X)$ is $\xi/2$-private. 
Assume $|f(X) - f(Y)|$ is at most $\eta_f$. Then, $\hat{f}(X)$ is  $\xi/2$-private as well. 
Since privacy preserved after post-processing (Lemma \ref{lem:post_processing}), 
both $E_1(X)$ and $E_2(X)$ are  $\xi/2$-private. 
Using the composition lemma \ref{lem:composition}, 
the output is $\xi$-private (by Lemma \ref{lem:composition}).

Now, assume $|f(X) - f(Y)|$ is greater than $\eta_f$. 
In this case, we show that $n_{\max}(X)$ has to be large. 
Therefore, the output is \reject with high probability regardless of $\hat f(X)$. 
Although $\hat f(X)$ is not private, it cannot affect the output drastically 
and the output remains private. We prove this formally below. 
Without loss of generality, assume we replace a sample $i$ in $X$ with $j$ to get $Y$. 
Thus, we have
$$
\begin{array}{ll}
\left| f(X) - f(Y) \right|& = \left|{n_i(X) \choose 2} + {n_j(X) \choose 2} - {n_i(Y) \choose 2} - {n_j(Y) \choose 2} \right|
\vspace{2mm} \\
& = \left| {n_i(X) \choose 2} + {n_j(X) \choose 2} - {n_i(X)-1 \choose 2} - {n_j(X) + 1 \choose 2}\right|
\vspace{2mm} \\
& = \left| n_i(X) - 1 - n_j(X) \right|
\vspace{2mm} \\
& \leq n_{\max} (X) \;,
\end{array}
$$
where the inequality comes from the assumption that there is at least one copy of $i$ in $X$. 
Therefore, $n_{\max}(X)$ is greater than $\eta_f$ as well. 
Since $T$ is even smaller than $\eta_f$, it is very unlikely that 
$\hat n_{\max}$ be smaller than the threshold $T$. 
More formally, by the properties of the Laplace distribution, we have:
\begin{equation} \label{eq:accept_bound}
\begin{array}{ll}
\Pr[E_1(X)] & = \Pr[\hat{n}_{\max}(X) \leq T] \leq \Pr[ \hat{n}_{\max}(X) - n_{\max}(X)  \leq T - \eta_f] 
\vspace{2mm}\\
& \leq \Pr\left[ \Lap(2/\xi) \leq - \frac { 2 \max(\ln 3, \ln 3/\xi)}{\xi}\right] 
\vspace{2mm}\\
& \leq  {\exp\left( - \max(\ln 3, \ln 3/\xi) \right)}/2 \leq \min(1/6, \xi/6) \;.
\end{array}
\end{equation}
Now, consider the case that the algorithm output \accept on input $X$.  It is not hard to see that 
\begin{equation}\label{eq:accValue}
\begin{array}{rl}
\Pr\left[\MM(X) = \mbox{\accept}\right]  & = (5/6) \cdot \Pr[E_1(X) \wedge E_2(X)] + (1/6) \cdot (1 - \Pr[E_1(X) \wedge E_2(X)])
\vspace {2mm} \\
& = (2/3) \cdot \Pr[E_1(X) \wedge E_2(X)] + 1/6
\vspace {2mm} \\
& = (2/3) \cdot \Pr[E_1(X)] \cdot \Pr[E_2(X)] + 1/6 \;.
\end{array} 
\end{equation}
Observe that since we flip the answer with probability $1/6$ at the end, 
$\Pr[\MM(X) = \mbox{\accept}]$ and $\Pr[\MM(Y) = \mbox{\accept}]$ are at least $1/6$. 
By this fact, Equation (\ref{eq:accept_bound}), and Equation (\ref{eq:accValue}), we have:
$$
\dfrac {\Pr[\MM(X) = \mbox{\accept}]}{\Pr[\MM(Y) = \mbox{\accept}]}   \leq \dfrac {\Pr[E_1(X)]  + 1/6}{1/6} \leq \xi + 1 < e^{\xi} \;. 
$$
Now, consider the case where the output of the algorithm is \reject on the input $X$.
Similar to Equation (\ref{eq:accept_bound}), we can prove $\Pr[E_1(Y)]$ is at most $\min(1/6, \xi/6)$. 
Similar to Equation (\ref{eq:accValue}), it is not hard to see that 
\begin{equation}\label{eq:rejValue}
\Pr\left[\MM(X) = \mbox{\reject}\right]  
 = (2/3) \cdot (\Pr[\overline {E_1}(X) \vee  \overline{E_2}(X)]) + 1/6 \;.
\end{equation}
If $\Pr[\MM(X) = \mbox{\reject}]$ is at most $\Pr[\MM(Y) = \mbox{\reject}]$, then clearly, we have:
$$\dfrac {\Pr[\MM(X) = \mbox{\reject}]}{\Pr[\MM(Y) = \mbox{\reject}]} \leq 1 < e^\xi \;.$$
Thus, assume $\Pr[\MM(X) = \mbox{\reject}]$ is less than $\Pr[\MM(Y) = \mbox{\reject}]$. 
Then, we have:
$$
\begin{array}{rl}
\dfrac {\Pr[\MM(X) = \mbox{\reject}]}{\Pr[\MM(Y) = \mbox{\reject}]}   
& = \dfrac{(2/3) \cdot (\Pr[\overline {E_1}(X) \vee  \overline{E_2}(X)]) + 1/6}{(2/3) \cdot (\Pr[\overline {E_1}(Y) \vee  \overline{E_2}(Y)]) + 1/6}
\vspace{2mm}\\
& \leq \dfrac{\Pr[\overline {E_1}(X) \vee  \overline{E_2}(X)]}{\Pr[\overline {E_1}(Y) \vee  \overline{E_2}(Y)] } \leq  \dfrac 1 {1 - \Pr[E_1(Y)]}
\vspace{2mm}\\
& \leq \dfrac 1 {1 - \min(1/6, \xi/6)} < 1 + \xi < e^\xi \;.
\end{array} 
$$
The second to last inequality is true since we showed previously that $\Pr[E_1(Y)]$ is at most $\min(1/6, \xi/6)$. 
Hence, the proof is complete. 
\end{proof}

\section{Proof of Theorem \ref{thm:closeness}} \label{sec:proof_closeness}

\thmCloseness*

\begin{proof}
Our proof has two main parts. 
First, we show that the algorithm outputs the correct answer with probability $2/3$. 
Second, we show that the algorithm is private. 

\noindent
\textbf{Proof of Correctness:} 
First, assume $p$ and $q$ are equal. 
In the algorithm, we compute $Z$ and add Laplace noise, $\eta$, to it. 
Then we compare it to threshold $T \coloneqq \epsilon^2 m^2/(8n + 4m)$.
Based on Equation (\ref{eq:EZ}), we have

$$\E[Z'] = \E[Z] + \E[\eta] = \E[Z].$$
Using the Chebyshev inequality and Equation (\ref{eq:VarZ}), 

$$\Pr[\mbox{outputting \reject}] = \Pr[Z' > T] \leq \dfrac{\Var[Z']}{T^2} \leq \dfrac{\Var[Z] + \Var[\eta]}{T^2} 
\leq \dfrac{2\min\{m, n\} + 128/\xi^2}{T^2} \leq \frac 1 3 \;,$$
where the last inequality is true for a sufficiently large universal constant $C$.

\smallskip

\textbf{Case 1:} Consider the case $m \leq n$. Then,

$$\dfrac{2\min\{m, n\} }{T^2} = \dfrac {2\, m\,(8n + 4m)^2  }{m^4 \epsilon^4} \leq \dfrac {2\, m\,(12n)^2  }{m^4 \epsilon^4} \leq 288 \, \left( \dfrac {n^{2/3}  }{ \epsilon^{4/3}} \cdot \frac 1 m\right)^3 \leq \frac{288}{C^3} \leq \frac 1 6 \;,$$
where the last inequality is true for $C$ greater than $12$. Moreover, 

$$\dfrac{128 \, \xi^{-2}}{T^2} \leq \dfrac{128\,(8n+4m)^2}{\xi^2 \, m^4 \, \epsilon^4}  
\leq \dfrac{128\,(12 \, n)^2}{\xi^2 \, m^4 \, \epsilon^4} 
\leq 18432 \left(\dfrac{\sqrt n}{\sqrt\xi \epsilon} \cdot \dfrac 1 m \right)^4 \leq \dfrac {18432}{C^4} \leq \frac 1 6 \;,
$$

where the last inequality is true for $C$ greater than $19$. 
Thus, for sufficiently large $C$, 
the probability of rejecting two identical distribution $p$ and $q$ is less than $1/3$.
\\
\textbf{Case 2:} Consider the case $n < m$. Then,

$$\dfrac{2\min\{m, n\} }{T^2} = \dfrac {2\, n\,(8n + 4m)^2  }{m^4 \epsilon^4} \leq \dfrac {2\, n\,(12m)^2  }{m^4 \epsilon^4} \leq 288 \, \left( \dfrac {\sqrt n  }{ \epsilon^2} \cdot \frac 1 m\right)^2 \leq \frac{288}{C^2} \leq \frac 1 6 \;,$$
where the last inequality is true for $C$ greater than $42$. 
Moreover, 

$$\dfrac{128 \, \xi^{-2}}{T^2} \leq \dfrac{128\,(8n+4m)^2}{\xi^2 \, m^4 \, \epsilon^2}  \leq \dfrac{128\,(12 \, m)^2}{\xi^2 \, m^4 \, \epsilon^2} \leq 18432 \left(\dfrac{1}{\xi \, \epsilon} \cdot \dfrac 1 m \right)^2 \leq \dfrac {18432}{C^2} \leq \frac 1 6 \;,$$
where the last inequality is true for $C$ greater than $136$. 
Thus, for sufficiently large $C$ the probability of rejecting two identical distribution $p$ and $q$ is less than 1/3. \\
\\
Now, suppose $p$ and $q$ are at least $\epsilon$-far from each other in $\ell^1$-distance. 
We show that in this case $Z'$ is greater than $T$ with high probability using Chebyshev's inequality. 
Based on Equation (\ref{eq:VarZ}), we bound the variance of $Z'$ in terms of the expected value of $Z'$. 
First, observe that, by Equation (\ref{eq:EZ}), we have that $\E[Z']$ 
is at least $C/6$ for any setting of parameters. Thus, for sufficiently large $C$, we can assume $\E[Z']$ is at least 360.    
Let $I_1$ be the set of all indices $i$ such that 
$(1 - (1- e^{-m(p_i + q_i)})/(m(p_i + q_i)))$ is greater $1/2$, 
and let $I_2$ be the set of remaining indices, i.e., $I_2 = [n] \setminus I_1$. 
By Equation (\ref{eq:EZ}), we have
\begin{align*}
\E[Z']^2 &= \left(\sum\limits_i \dfrac{(p_i - q_i)^2}{p_i + q_i}m \left(1 - \dfrac{1 - e^{-m(p_i+q_i)}}{m(p_i+q_i)}\right)\right)^2
\geq 
360 \, \sum\limits_i \dfrac{(p_i - q_i)^2}{p_i + q_i}m \left(1 - \dfrac{1 - e^{-m(p_i+q_i)}}{m(p_i+q_i)}\right)
\vspace{2mm}\\
& \geq 
360 \, \sum\limits_{i \in I_1} \dfrac{(p(i) - q(i))^2}{p(i) + q(i)}m \left(1 - \dfrac{1 - e^{-m(p(i)+q(i))}}{m(p(i)+q(i))}\right)
\geq 36 \sum\limits_{i \in I_1} 5m\dfrac{(p(i) - q(i))^2}{p(i) + q(i)} \;.
\end{align*}
On the other hand, for any $i$ in $I_2$, we can conclude that $m(p(i) + q(i))$ is less than $2$. 
Therefore, $m \frac{(p(i) - q(i))^2}{p(i) + q(i)}$ is at most 2. Thus, $\sum\limits_{i \in I_2} 5m \frac{\left(p(i) - q(i)\right)^2}{p(i) + q(i)}$ is at most $10\, n$. Since $\frac{\left(p(i) - q(i)\right)^2}{p(i) + q(i)}$ is less than $p(i) + q(i)$,  
$\sum\limits_{i \in I_2} 5m \frac{\left(p(i) - q(i)\right)^2}{p(i) + q(i)}$ is also less than $10 \, m$. 
Hence, we have
\begin{align*}
\Var[Z] & \leq 2\min\{m, n\} + \sum\limits_i 5m \dfrac{\left(p(i) - q(i)\right)^2}{p(i) + q(i)}
\vspace{2mm}\\
& \leq 2\min\{m, n\} + \sum\limits_{i \in I_1} 5m \dfrac{\left(p(i) - q(i)\right)^2}{p(i) + q(i)} + \sum\limits_{i \in I_2} 5m \dfrac{\left(p(i) - q(i)\right)^2}{p(i) + q(i)}
\vspace{2mm}\\
& \leq 12\min\{m, n\} + \dfrac { \E[Z']^2}{36}  \;.
\end{align*}
By Equation (\ref{eq:EZ}), the expected value of $Z'$ is at least $2T$.  
Using Chebyshev's inequality, we obtain
\begin{align*}
\Pr[\mbox{outputting ``Accept''}] 
& = \Pr[Z' \leq T] \leq \Pr[ \E[Z'] - Z' \geq \E[Z'] - T] \leq \Pr\left[ \E[Z'] - Z' \geq \dfrac{\E[Z'] } 2\right]
\vspace{2mm}\\
& \leq \dfrac{4\Var[Z']} {\E[Z']^2} \leq \dfrac{4(\Var[Z] + \Var[\eta])}{\E[Z']^2} \leq \dfrac{48 \min\{m, n\}}{\E[Z']^2} + \dfrac {10}\AA + \dfrac {512}{\E[Z']^2 \, \xi^2}
\vspace{2mm}\\
& \leq  \dfrac{48 \min\{m, n\}(4n + 2m)^2}{m^4 \epsilon^4} + \dfrac 1 9 + \dfrac {512(4n + 2m)^2}{m^4 \, \epsilon^4 \, \xi^2}  \leq \dfrac 1 3 \;,
\end{align*}
where the last inequality is true for sufficiently large $C$.
\medskip

\noindent
\textbf{Proof of Privacy Guarantee:} 
First, observe that the value of $Z$ does not change drastically over two neighboring datasets. 
More formally, we have the following simple lemma:
\begin{lemma}
The sensitivity of the statistic $Z$ is at most 8. 
\end{lemma}
\begin{proof}
Assume two neighboring dataset $x$ and $y$.
 Let $Z^{(x)}$ and $Z^{(y)}$ be the statistic for $x$ and $y$ respectively.  We define $Z_i$ as follows:
 
$$
Z_i \coloneqq \left\{\begin{array}{ll}
\dfrac{|X_i + Y_i| - X_i - Y_i}{X_i + Y_i} & \quad \quad \mbox{if }X_i + Y_i \not = 0
\vspace{2mm}\\
0& \quad \quad \mbox{otherwise.}\end{array}
\right.
$$
We use a superscript $(x)$ or $(y)$ for $X_i$, $Y_i$, $Z_i$ 
to indicate the corresponding dataset we calculate them from. 
Since $x$ and $y$ are two neighboring datasets, 
there is a sample $i$ in the $x$ which has been replaced by $j$. 
Without loss of generality, assume $i$ was a sample from $p$. 
This implies that $X_i^{(x)} - X_i^{(y)} = 1$ and $Y_i^{(x)} = Y_i^{(y)}$. 

If $X_i^{(y)} + Y_i^{(y)}$ is zero, then $Z_i^{(x)}$ is one. 
Thus, the difference of $Z^{(x)}_i$ and $Z^{(y)}_i$ is one. 
Now, assume $X_i^{(y)} + Y_i^{(y)}$ is at least one. Then, we have
\begin{align*}
 \left|Z_i^{(x)} - Z_i^{(y)}\right| 
& =  \left|\dfrac{\left(X_i^{(x)} - Y_i^{(x)}\right)^2 }{X_i^{(x)} + Y_i^{(x)}} - \dfrac{\left(X^{(y)}_i - Y^{(y)}_i\right)^2}{X^{(y)}_i + Y^{(y)}_i}\right| 
\vspace{2mm}\\
& =  \left|\dfrac{\left(X_i^{(y)} - Y_i^{(y)} + 1\right)^2 }{X_i^{(y)} + Y_i^{(y)} + 1 } - \dfrac{\left(X^{(y)}_i - Y^{(y)}_i\right)^2}{X^{(y)}_i + Y^{(y)}_i}\right| 
\vspace{2mm}\\
& =  \left|\dfrac{\left(X_i^{(y)} - Y_i^{(y)}\right)^2 + 2 \left(X_i^{(y)} - Y_i^{(y)}\right) + 1 }{X_i^{(y)} + Y_i^{(y)} + 1 } - \dfrac{\left(X^{(y)}_i - Y^{(y)}_i\right)^2}{X^{(y)}_i + Y^{(y)}_i}\right| 
\vspace{2mm}\\
& =  \left|\dfrac{2 \left(X_i^{(y)} - Y_i^{(y)}\right) + 1 }{X_i^{(y)} + Y_i^{(y)} + 1 } - \dfrac{(X^{(y)}_i - Y^{(y)}_i)^2}{\left(X_i^{(y)} + Y_i^{(y)} + 1 \right) \cdot \left(X^{(y)}_i + Y^{(y)}_i\right)}\right| 
\vspace{2mm}\\
& \leq 2 + 1 + 1 \leq 4 \;.
\end{align*} 
Similarly, we can show $|Z_j^{(x)} - Z_j^{(y)}|$ is at most four. 
Hence, we can conclude that $|Z^{(x)} - Z^{(y)}|$ is at most eight.
\end{proof}

Therefore, using the property of the Laplace mechanism (Lemma \ref{lem:Laplace_mechanism}), 
$Z'$ is $\xi$-private. Using Lemma \ref{lem:post_processing} and the fact that the output 
of the algorithm is a function of $Z'$, we  conclude the algorithm is $\xi$-private.
\end{proof}

\end{document}